\documentclass[10pt]{ietbook}

\usepackage[ansinew]{inputenc}
\usepackage{color,placeins}
\usepackage{hyperref,cite}
\usepackage{graphicx,amsmath}
\usepackage{graphics} 
\usepackage{mathptmx} 
\usepackage{times} 
\usepackage{amsmath} 
\usepackage{amssymb}  
\usepackage{amsfonts}
\usepackage{amsmath}
\usepackage{amssymb}
\usepackage{psfrag}
\usepackage{float}

\def\bi{\begin{itemize}}
\def\ei{\end{itemize}}
\def\bn{\begin{enumerate}}
\def\en{\end{enumerate}}
\def\bq{\begin{eqnarray}}
\def\eq{\end{eqnarray}}
\def\be{\begin{equation}}
\def\ee{\end{equation}}
\def\bea{\begin{eqnarray}}
\def\eea{\end{eqnarray}}
\def\beann{\begin{eqnarray*}}
\def\eeann{\end{eqnarray*}}
\def\bsea{\begin{subeqnarray}}
\def\esea{\end{subeqnarray}}
\def\bmat{\left[ \begin{array}}
\def\emat{\end{array} \right]}
\newfont{\BB}{msbm10}
\newfont{\bb}{msbm8}

\def\sp{{\rm span}\ }

\newcommand{\bmx}{\begin{matrix}}
\newcommand{\emx}{\end{matrix}}
\newcommand{\ba}{\begin{array}}
\newcommand{\ea}{\end{array}}

\def\nn{\nonumber}
\def\bq{\begin{eqnarray}}
\def\eq{\end{eqnarray}}

\def\bsmat{\left[ \begin{smallmatrix}}
\def\esmat{\end{smallmatrix} \right]}

\begin{document}


\rhbooktitle{Swarm Intelligence - From Concepts to Applications}

\markboth{Swarm Intelligence - From Concepts to Applications}{Bounded Distributed Flocking Control of  Nonholonomic Mobile Robots}

\cauthor{Thang Nguyen\thanks{Advanced Robotics and Automation (ARA) Lab, Department of Computer Science and Engineering, University of Nevada, Reno, NV89557, USA (e-mail: thangnthn@gmail.com)},
Hung M. La\thanks{Advanced Robotics and Automation (ARA) Lab, Department of Computer Science and Engineering, University of Nevada, Reno, NV89557, USA (e-mail:hla@unr.edu)},
Vahid Azimi\thanks{School of Electrical and Computer Engineering, Georgia Institute of Technology, 777 Atlantic Drive NW
Atlanta, GA 30332-0250, USA (e-mail: vahid.azimi@gatech.edu)}, and
Thanh-Trung  Han\thanks{Faculty of Electrical and Electronics Engineering, Ton Duc Thang University, Ho Chi Minh City, Vietnam (e-mail: hanthanhtrung@tdt.edu.vn)}
}

\chapter{Bounded Distributed Flocking Control of Nonholonomic Mobile Robots}

There have been numerous studies on the problem of flocking control for multiagent systems whose simplified models are presented in terms of point-mass elements. Meanwhile, full dynamic models pose some challenging problems in addressing the flocking control problem of mobile robots due to their nonholonomic dynamic properties. Taking practical constraints into consideration, we propose a novel approach to distributed flocking control of nonholonomic mobile robots by bounded feedback. The flocking control objectives consist of velocity consensus, collision avoidance, and cohesion maintenance among mobile robots. A flocking control protocol which is based on the information of neighbor mobile robots is constructed. The theoretical analysis is conducted with the help of a Lyapunov-like function and graph theory. Simulation results are shown to demonstrate the efficacy of the proposed distributed flocking control scheme.


\section{INTRODUCTION}
\label{introduction}
\noindent

It is well-known that the collective behavior of self-propelled organisms constitutes flocking \cite{TonerTu98}. The coherent motion of the flock inspires various research on flocking control of multiagent systems. A typical objective is to achieve a desired collective motion which can be produced by a constructive flocking control procedure.  For numerous models, which are described from simplest models such as point-mass models to actual physical models, design protocols have been systematically proposed for multiagent systems \cite{La_TC_2013, Filiberto_RNC3687}. Several control strategies were also addressed  in noisy environments where the agent's position is affected by noise \cite{La_ICRA10, La_JC,Dang_MFI}. With point-mass models, the problem of flocking control of multiple agents has been addressed with typical results reported in \cite{Saber06,tanner2007flocking,CuckerSmale07,DongHuang15,ghapani2016fully}. For a wide range of engineering applications, extensive studies in flocking control of mobile robots have been done in various scenarios \cite{moshtagh2007distributed,HanGe15, La_SMCA_2015}.

In this chapter, we study the problem of distributed flocking control of mobile robots by bounded feedback, which takes into consideration nonholonomic nature of mobile robots as well as the implementation issue posed by the physical limit of the motor speed. Our flocking control problem employs the full dynamic model of the mobile robot derived in \cite{AmarMohamed13}. Similar to \cite{Han2016,Nguyen2016}, due to the nonholonomic property of the dynamics of mobile robots, our proposed design framework constructed to achieve velocity consensus is modular. In other words, the consensuses on the linear speed and orientation angles are obtained separately. 

In this chapter, we are interested in agents with nonholonomic dynamics and boundedness constraints. Specifically, a coordination function is proposed to ensure that the induced attractive and repulsive forces are bounded, and hence can be incorporated in the bounded velocity control. Using the results of Barbalat's lemma and graph theory, the theoretical analysis is conducted, which shows that the maximal value of the coordination function determines the basin of attraction for the flocking convergence.

In this chapter, graph theory will be employed as in the case of nearest neighbor communication \cite{JadbabaieLinMorse03, Saber06}. We will employ the velocity control law reported in \cite{Han2016,Nguyen2016} in a decentralised sense, which helps to avoid collision and maintain a linear speed consensus. In addition, the orientation consensus will be achieved using a modified approach, which is inspired by the one in \cite{LiangLee2005}, where the input constraint is taken into account.

The organization of the chapter is as follows. Section\ref{relatedwork} summarises some research work in the literature related to the topic in this chapter. In Section \ref{problemformulation}, the multiple-goal control problem for flocking of nonholonomic mobile robots and preliminaries are introduced.  Section \ref{mainresults} describes main results where a modular design framework is proposed for bounded velocity control and bounded orientation control and the theoretical analyses are introduced. In Section \ref{avoidance_obstacles}, a description of an obstacle avoidance scheme is presented. Section \ref{simulation} shows some simulation results. Section \ref{conclusions} concludes the chapter by some conclusions.

\emph{Notations:} ${\mathbb R}$ and ${\mathbb R}^+$ are the sets of real numbers and nonnegative real numbers, respectively; for $q = [q_1, \ldots, q_n]^T$, $\nabla_q = [\partial/\partial q_1, \ldots, \partial/\partial q_n]^T$ is the del operator \cite{RileyHobsonBence2006}; for two vectors $a$ and $b$, $a \cdot b$ is their scalar product; $ (a_1,\ldots, a_n)$ is $[a_1^T,\ldots,a_n^T]^T$; $|\cdot|$ is the absolute value of scalars; and $\|\cdot\|$ is the Euclidean norm of vectors.

\section{RELATED WORK}
\label{relatedwork}

In many applications, the mission carried out by a single complicated robotic system can be equivalently completed by a coordination of a mobile robotic system with much simpler configurations, whose advantages lie in scalability, flexible deployment, cheaper costs, reliability, etc. Therefore, more sophisticated tasks can be fulfilled using a group of small mobile robots with lower cost and higher efficiency than a complex unit; see \cite{La_TC_2013,TannerJadbabaiePappas04,LiangLee2006,Nguyen2014,La_SMCA_2015,nguyen2015formation,
cruz2007decentralized,nguyen2017distributed,pham2017distributed,LaShengICRA2009,LaShengACC2009,LaShengIROS2009,
NguyenLaTeague2015,LaShengChen2013,la2011cooperative,LaSMC2009,la2010hybrid,LaCTACIS2013,Jafari2015,la2015multi,
la2011flocking} and references therein.

Flocking control of mobile robots was widely addressed with different control schemes; see \cite{JadbabaieLinMorse03, CuckerSmale07, carrillo2010asymptotic,LaLimSheng15,TannerJadbabaiePappas04, SepulchrePaleyLeonard07} and references therein. Recently, a new measure-theoretic approach which systematically provides a framework for obtaining flocking protocols for mobile robots was reported in \cite{HanGe15}. 

The common assumption in many papers is the availability of information of all agents or all-to-all communication. Numerous control protocols for mobile robots have been constructed based on this assumption. This centralized communication control architecture yields  inflexibility and large computation costs for the controller of each agent, especially when the number of agents is large. Meanwhile, a distributed control protocol can offer an ease of implementation and less computational burden as each element of the system needs only the information of neighbor agents. In this direction, a range of decentralized control schemes for mobile robots have been proposed \cite{JadbabaieLinMorse03,tanner2007flocking,zavlanos2009hybrid,dong2011flocking,DongHuang15}.

For a wide range of engineering applications, cohesion maintenance and collision avoidance (CMCA) properties of a mobile robotic system are of importance. As reported in \cite{TannerJadbabaiePappas04, MoshtaghMichaelJadbabaieDaniilidis09}, the attractive and repulsive forces cannot be included in the control for CMCA of mobile robots, as it is possible for point-mass agents \cite{Saber06}. In \cite{HanGe15}, desired attractive and repulsive forces for CMCA of mobile robots was achieved using a new rearrangement strategy. In \cite{JadbabaieLinMorse03, Saber06,DongHuang15}, the graph theory was employed to generate control protocols that maintain CMCA of multiagent systems with double integrator models.

In \cite{LiangLee2005}, a distributed flocking control approach was proposed but no constraints on the control inputs are imposed. The work in \cite{Han2016,Nguyen2016} considers the bounded feedback flocking control problem for nonholonomic mobile robots without a flocking desired heading angle. The problem of interest in this chapter is to address bounded control of nonholonomic dynamic mobile vehicles, which also achieves CMCA and obstacle avoidance. We also consider a flocking desired heading angle, which reveals a collective flocking behaviour.

\section{PROBLEM FORMULATION}
\label{problemformulation}

Similarly to \cite{Han2016,Nguyen2016}, we investigate a collective system of $N$ identical autonomous mobile robots whose respective equations of motion are \cite{AmarMohamed13}
\begin{align}
\label{mobiledynamics}
  \dot q_i & = v_ie(\theta_i) \nonumber\\
  \dot \theta_i & = w_i \nonumber\\
  \dot v_i & = u_i \nonumber\\
  \dot w_i & = \tau_i
\end{align}
where  $i=1, . . . , N$, $q_i = [x_i, y_i]^T \in \mathbb{R}^2$, and $\theta_i \in \mathbb{R}$
are respectively the position and the heading angle of the $i$-th robot in the inertial frame $Oxy$; $v_i \in \mathbb{R}$ is the linear speed, and $e(\theta_i)$ is the unit vector $[\cos\theta_i, \sin \theta_i]^T$ ; $w_i \in \mathbb{R}$ is the angular speed, and $u_i, \tau_i \in \mathbb{R}$ are control inputs.

Following the same vein as in \cite{Han2016}, we define $0<r_0<R_0$. Then, the flocking control problem for (\ref{mobiledynamics}) is to construct the control inputs $u_i,\,\tau_i$ as bounded functions of the collective state $(q_1, \ldots, q_N, \theta_1, \ldots, \theta_N, v_1, \ldots, v_N, w_1, \ldots, w_N)$ in a distributed fashion to satisfy the following multiple goals 
\begin{itemize}
  \item[G1)] \emph{Velocity consensus:}
  \begin{align}
    \lim_{t\to\infty}(\dot q_i(t) - \dot q_j(t)) = 0, \forall i,j = 1, \ldots, N
  \end{align}
  \item[G2)] \emph{Collision avoidance:} $r_{ij}(t) = \|q_i(t) - q_j(t)\| \ge r_0, \forall t \ge 0, \forall i \ne j$
  \item[G3)] \emph{Cohesion maintenance:} $r_{ij}(t) \le R_0, \forall t \ge 0, \forall i \ne j$.
\end{itemize}

Similarly to \cite{Han2016,Nguyen2016}, we have the following definition.
\begin{definition}
\label{QAVbc.bcdef}
 A control $\dot \zeta  = g(\zeta,y),   u  = c(\zeta,y), (\zeta,y) \in \mathbb{R}^d\times \mathbb{R}^m$ of a system $\dot x  = f(x,u), y = h(x,u)$ is said to be bounded if there is a finite constant $M > 0$ such that $\|c(\zeta,y)\| \le M, \forall (\zeta,y) \in \mathbb{R}^d\times \mathbb{R}^m$.
\end{definition}

To achieve the goals G2) and G3), we consider the coordination function 
$U:\mathbb{R}^+\to\mathbb{R}^+$ which satisfies the following properties:
\begin{itemize}
  \item[P1)] there is a constant $U_M > 0$ such that
  \begin{align} 0 \le U(r) \le U_M, \forall r \in \mathbb{R}
  \end{align}
  \item[P2)] $U(r)$ is continuously differentiable on $[r_0, R_0]$;
  \item[P3)] $\lim\limits_{r \to r_0^+}U(r) = U_M$; and
  \item[P4)]$\lim\limits_{r \to R_0^-} U(r) = U_M$.
\end{itemize}

For a link between agents $i$ and $j$ of the flock, we aim to maintain $r_{ij}(t) \in [r_0, R_0]$. Without loss of generality, we assume that $U(0) = 0$ and hence $U(r)$ is well defined for $r_{ii} = 0$ \cite{Han2016}.

We are interested in the function $U$ with the dead zone $[a,A]$ since even distribution of agents may not be achievable by a common coordination function $U$. Accordingly, we use the zone $[a,A]$ for free alignment. A function $U$ satisfying the above requirements is shown in Figure \ref{flk.sim.ecflocking}.

\begin{figure}[!t]
\includegraphics[width = \columnwidth]{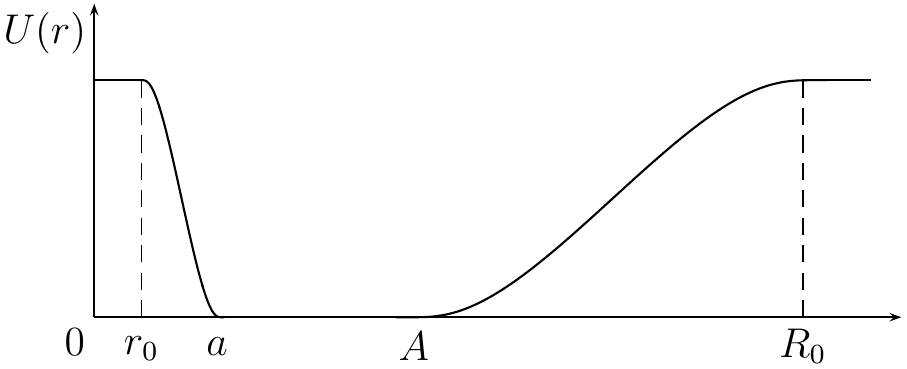}
\caption{Coordination function (extracted from \cite{Han2016}).} \label{flk.sim.ecflocking}
\end{figure}


For bounded control, we shall use the  linear saturation functions  $\sigma_1$, and $\sigma_2$, which are continuous and nondecreasing functions and satisfy, for given positive constants $L_i \le M_i, i = 1,2$ \cite{Han2016}
\begin{itemize}
  \item[i)] $\sigma_i(-s) = -\sigma_i(s)$ for all $s$;
  \item[ii)] $\sigma_i(s) = s$ for $s \le L_i$; and
  \item[iii)] $|\sigma_i(s)| \le M_i, \forall s \in \mathbb{R}$.
\end{itemize}

Similarly to other works on distributed for multiagent systems  \cite{JadbabaieLinMorse03, Saber06,DongHuang15,Nguyen2016}, the graph theory will be utilised to address our problem. A digraph associated with (\ref{mobiledynamics}) is called $\mathcal{G}(t)=(\mathcal{V},\mathcal{E}(t))$ where $\mathcal{V}={1,\dots, N}$ and $\mathcal{E}\subseteq \mathcal{V} \times \mathcal{V}$. The set $\mathcal{V}$ is denoted as the node set of $\mathcal{G}(t)$ and the set $\mathcal{E}(t)$ is defined as the edge set of $\mathcal{G}(t)$. In addition, $\mathcal{N}_i(t)$ denotes the neighbor set of the node $i$ for $i=1, \dots, N$.

As in \cite{DongHuang15}, the description of the edge $\mathcal{E}(t)$ is presented as follows.

Given any $R>0$, $\epsilon_2 \in (0,R)$, and $\epsilon_1 \in (0,R-\epsilon_2)$, for any $t\geq 0$, $\mathcal{E}(t)=\{(i,j) |i,j \in \mathcal{V}\}$ is defined such that 
\begin{enumerate}
\item $\mathcal{E}(0)=\{(i,j) | \epsilon_1 < \|q_i(0) -q_j(0)\| < (R-\epsilon_2)\}$;
\item if $\|q_i(0) -q_j(0)\| \geq R$, then $(i,j)\notin \mathcal{E}(t)$;
\item for $i= 1,\dots, N$, $j=1, \dots, N$, if $(i,j)\notin \mathcal{E}(t^-)$ and $\|q_i(t) -q_j(t)\| < R-\epsilon_2$, then $(i,j) \in \mathcal{E}(t)$;
\item for $i= 1,\dots, N$, $j=1, \dots, N$, if $(i,j)\in \mathcal{E}(t^-)$ and $\|q_i(t) -q_j(t)\| < R$, then $(i,j) \in \mathcal{E}(t)$.
\end{enumerate}

As in \cite{Nguyen2016}, the following results will be employed for the main results.
\begin{lemma}
\label{lemmasum}
  Let $\sigma:\mathbb{R} \to \mathbb{R}$ be a function satisfying $\sigma(-s) = -\sigma(s), \forall s \in \mathbb{R}$. Then, for all $a_i, b_i$, it holds true that
  \begin{align}
  \label{full.lemsum}
    \frac{1}{2}\sum_{i=1}^N\sum_{j\in \mathcal{N}_i(t)}(a_i - a_j)\sigma(b_i - b_j) = \sum_{i=1}^N\sum_{j\in \mathcal{N}_i(t)}a_i\sigma(b_i - b_j).
  \end{align}
  \vspace{0pt}
\end{lemma}
\begin{proof}
 Since $\sigma(-s) = -\sigma(s)$ and $\mathcal{G}(t)$ is an undirected graph, we have
  \begin{align}
   \sum_{i=1}^N\sum_{j\in \mathcal{N}_i(t)}  a_j\sigma(b_i - b_j) & = - \sum_{i=1}^N\sum_{j\in \mathcal{N}_i(t)}  a_j\sigma(b_j - b_i) \nonumber\\
   & = - \sum_{i=1}^N\sum_{j\in \mathcal{N}_i(t)} a_i\sigma(b_i - b_j).
  \end{align}
  Hence,
  \begin{align}
    &\sum_{i=1}^N\sum_{j\in \mathcal{N}_i(t)}(a_i - a_j)\sigma(b_i - b_j) \nonumber\\
    & \quad = \sum_{i=1}^N\sum_{j\in \mathcal{N}_i(t)}a_i \sigma(b_i - b_j) - \sum_{i=1}^N\sum_{j\in \mathcal{N}_i(t)}a_j\sigma(b_i - b_j) \nonumber\\
    & \quad = 2\sum_{i=1}^N\sum_{j\in \mathcal{N}_i(t)}a_i\sigma(b_i - b_j)
  \end{align}
  which implies (\ref{full.lemsum}).
\end{proof}	

\begin{remark}
Lemma \ref{lemmasum} plays an important role in the theoretical analysis of the main results. Here, the lemma is similar to the one in \cite{Han_CASE}. However,  \cite{Han_CASE} considers all-to-all communication in the multiagent system. Our problem in this chapter is focused on the distributed fashion, which requires the employment of the neighbour set $\mathcal{N}_i(t)$ of robot $i$.
\end{remark}

\begin{lemma}\label{lemma2}
  The linear saturation functions $\sigma_i, i = 1, 2, 3$ satisfy
  \begin{align}
  \label{full.lemin0}
    (\sigma_i(\theta_1) - \sigma_i(\theta_2))\sigma_i(\theta_1 - \theta_2) \ge 0, \forall \theta_1, \theta_2.
  \end{align}
  \vspace{0pt}
\end{lemma}
\begin{proof}
  Without loss of generality, suppose that $\theta_1 \ge \theta_2$. Since $\sigma_i$ are nondecreasing functions, this implies that
  \begin{align}
  \label{full.lemin1}
    \sigma_i(\theta_1) - \sigma_i(\theta_2) \ge 0.
  \end{align}
  Furthermore, as $\sigma_i(0) = 0$, $\theta_1 \ge \theta_2$ and the nondecreasing property of $\sigma_i$ imply that 
  \begin{align}
  \label{full.lemin2}
    \sigma_i(\theta_1 - \theta_2) \ge 0.
  \end{align}
  Multiplying (\ref{full.lemin1}) and (\ref{full.lemin2}) side-by-side, we obtain (\ref{full.lemin0}).	
\end{proof}

\section{MAIN RESULTS}
\label{mainresults}

Our constructive strategy is to design $u_i$ to achieve consensus on $v_i$, and $\tau_i$ to achieve consensus on $\theta_i$. The design for $u_i$ is derived from \cite{Nguyen2016}, while the construction for $\tau_i$ is built based on the approach in \cite{LiangLee2005}.

Note that $U(r_{ij}) = U(\|q_i - q_j\|)$, which is the symmetric function of $q_i$ and $q_j$. As a result, we write $U(q_i, q_j)$ with the understanding that $U(q_i, q_j) = U(q_j, q_i)$. The control protocols $u_i$ and $\tau_i$ are constructed based on Lyapunov theory. Specifically, a positive definite function $V$ is presented such that the time derivative of $V$ is a negative definite function. Regarding the distribution control problem, the graph theory will be employed to show the connectivity preservation for our multiagent system. Then, we apply the LaSalle's invariance principle \cite{Khalil02} to conclude the desired consensuses.

Similarly to \cite{Nguyen2016}, the initial state of the collective system of agents (\ref{mobiledynamics}) is chosen such that the graph $\mathcal{G}(0)$ is connected. The parameters of the graph $\mathcal{G}(0)$  are chosen as follows
\bea
	R&=&R_0,\\
	r_0&\leq&\epsilon_1<a,\\
	0&<&\epsilon_2\leq R_0-a.
\eea

\subsection{Speed consensus and connectivity preservation}

The derivation of this subsection is essentially similar to the control design for the linear speed in \cite{Nguyen2016}; hence, it is presented here for completeness.
Consider the energy function for system (\ref{mobiledynamics})
\be
  V_1  =  \frac{1}{2}\sum_{i=1}^N\sum_{j\in \mathcal{N}_i(t)} U(q_i,q_j) + \frac{1}{2}\sum_{i=1}^Nv_i^2.
\ee

We assume that $U(r)$ is designed such that
\be
	U(R_0)=U_M> V_{1max},
\ee
where
\be
	V_{1max}\triangleq \frac{1}{2}\sum_{i=1}^N v^2(0)+\frac{N(N-1)}{2} U(R_0-\epsilon_2).
\ee

Let $m_0$ be the number of the links of the initial graph. The simplest connected graph of $N$ agents is a tree whose number of links is $n-1$. Hence, $m_0\geq n-1$. Let
\be\label{V10}
	V_1(0)\leq V_{1max}-\frac{(N-1)(N-2)}{2} U(R_0-\epsilon_2).
\ee

Note that $U(q_i,q_j)$ is a symmetric function of $q_i$ and $q_j$. We compute the derivative of $V_1$ with respect to (\ref{mobiledynamics})
\bea
\label{dotV11stcomp}
 \nn \dot V_1  &=&\sum_{i=1}^N\sum_{j\in \mathcal{N}_i(t)} \nabla_{q_i} U(q_i,q_j) \cdot \dot q_i + \sum_{i=1}^N v_i u_i \\  		&=&\sum_{i=1}^N v_i\Bigg(\sum_{j\in \mathcal{N}_i(t)} \nabla_{q_i} U(q_i,q_j)\cdot e(\theta_i)+u_i\Bigg).
\eea

From (\ref{dotV11stcomp}), a control law for the speed consensus protocol is chosen as 
\begin{align}
\label{proto.1sdesign.speed}
  u_i & = -\sum_{j\in \mathcal{N}_i(t)} \nabla_{q_i} U(q_i,q_j)\cdot e(\theta_i) - \sum_{j\in \mathcal{N}_i(t)} \sigma_1(v_i - v_j) 
\end{align}
where $\sigma_1$ is the linear saturation function introduced in Section \ref{problemformulation}.

Substituting (\ref{proto.1sdesign.speed}) into (\ref{dotV11stcomp}), we obtain
\begin{align}
\label{re.dotV1}
  \dot V_1 & = -\sum_{i=1}^N\sum_{j\in \mathcal{N}_i(t)} v_i \sigma_1(v_i - v_j).
\end{align}

We have the following speed consensus theorem \cite{Nguyen2016}.
\begin{theorem}\label{thms}
  Suppose that the collective system (\ref{mobiledynamics}) subject to the protocol (\ref{proto.1sdesign.speed}) is initiated such that $V_1(0) < V_{1max}$. Then, the following properties hold:
  \begin{itemize}
 \item[i)] $\mathcal{G}(t)$ is connected for all $t\geq 0$ and there exists $t_k$ such that for $t\geq t_k$, $\mathcal{G}(t)=\mathcal{G}(t_k)$
\item[ii)] Collision avoidance is guaranteed, i.e. $\|q_i-q_j\|>r_0$ for all $i,j\in N$ and $i\ne j$.
\item[iii)]$\lim\limits_{t\to\infty}(v_i(t) - v_j(t)) = 0$
  \end{itemize}
  \vspace{3pt}
\end{theorem}

\begin{proof}
Assume that $\mathcal{G} (t)$ switches at time $t_k$ $(k=1,2,\dots)$. Hence, $G(t)=G(0)$ for all $t\in [0, t_1)$. In other words,
\bea
\nn	\mathcal{G}(t)&=&G(0), \text{ } t\in[0, t_1)\\
	\mathcal{G}(t_1) &\ne& G(0).
\eea
We prove that $G(0) \subset G(t_1)$. Using control law (\ref{proto.1sdesign.speed}), we have
\be\label{dotV1}
	 \dot V_1= -\sum_{i=1}^N v_i \sum_{j\in \mathcal{N}_i(t)} \sigma_1(v_i - v_j). 
\ee

According to Lemma \ref{lemmasum}, 
\be
	\dot V_1= -\frac{1}{2}\sum_{i=1}^N  \sum_{j\in \mathcal{N}_i(t)} (v_i-v_j) \sigma_1(v_i - v_j). 
\ee

Because $\sigma_1(s)$ is an odd function, $(v_i-v_j) \sigma_1(v_i - v_j) \geq 0$. Thus, $\dot V_1 \leq 0$, which deduces
\be
	V_1(t) \leq V_1(0)< V_{1max} <U_M \text{ for } [0,t_1).
\ee
From the definition of $U(r)$, $U(R_0)>V_{1max}\geq V_1(0)$. Hence for any $(i,j) \in \mathcal{G}(t)$ for $t\in [0,t_1)$
\be\label{Up}
	U(q_i,q_j)\leq V_1(t)<U_M=U(r_0)=U(R_0).
\ee
By the continuity of $U(r)$, (\ref{Up}) shows that $r_0<\|q_i-q_j\|<R_0$. This implies that no existing links are deleted at time $t_1$ and collision avoidance is achieved. As a result, new links must be added to the current graph at the switching time $t_1$. We assume that there are $m_1$ new links being added to the network at time $t_1$. Since the number of the current links before switching is $m_0\geq N-1$ and the complete graph possesses $\frac{N(N-1)}{2}$ edges, $m_1 \leq \frac{N(N-1)}{2}-(N-1)=\frac{(N-1)(N-2)}{2}$. Hence, we have
\be
	V_1(t_1)=V_1(t_1^-)+m_1 U(R_0-\epsilon_2).
\ee
Due to (\ref{V10}),
\be
	V_1(t_1^-)\leq V_1(0) < V_{1max}-\frac{(N-1)(N-2)}{2} U(R_0-\epsilon_2).
\ee
Thus,
\be
	V_1(t_1)< V_{1max}.
\ee
By induction, for $t \in [t_{k-1}, t_k)$,
\be
	\dot V_1= -\frac{1}{2}\sum_{i=1}^N  \sum_{j\in \mathcal{N}_i(t)} (v_i-v_j) \sigma_1(v_i - v_j)\leq0, 
\ee
and therefore $V_1(t)\leq V_1(t_{k-1})\leq V_{1max}$. This shows that no edges are lost at time $t_k$ and $V_1(t_k)\leq V_{1max}$. As a result, the size of the set of the links of $\mathcal{G}(t)$ forms an increasing sequence, bounded above by $\frac{N(N-1)}{2}$, which is the number of the links of a complete graph. Thus, there exists a finite integer $k>0$ such that
\be\label{Gk}
	\mathcal{G}(t)=\mathcal{G}(t_k), \text{ } t\in [t_k,\infty).
\ee
Hence, for $t\geq t_k$, we have
\be\label{dotV12}
	\dot V_1= -\frac{1}{2}\sum_{i=1}^N  \sum_{j\in \mathcal{N}_i(t_k)} (v_i-v_j) \sigma_1(v_i - v_j)\leq 0. 
\ee

Next, we will show that the linear velocities of all agents converge to the same value. Using the fact that $U(q_i,q_j)\leq V_1(t) \leq V_{1max}<U_M$ and the properties of $U$, we deduce $\|q_i-d_j\|>r_0$. This shows that no collision takes place among agents. Since $0\leq V_1(t) \leq V_{1max}$ and $\dot V_1 \leq 0$, by Barbalat's lemma, $\lim_{t\rightarrow \infty} \dot V_1(t)=0$. Because the graph $\mathcal{G}(t)$ is connected for all $t$ and $s\sigma_1(s) \geq 0$ for all $s$, from (\ref{dotV12}),
\be
	\lim_{t\rightarrow \infty} (v_i - v_j)=0, \text{ for all } i,j=1,2, \dots, N.
\ee

\end{proof}

\begin{remark}
The proof of the theorem follows similar approaches as in \cite{ZavlanosCDC2007,DongHuang15}, where graph theory was employed as a means for proving the connectivity of mobile networks. The potential in this chapter is similar to the one in \cite{DongHuang15} in the sense that it is bounded. In contrast, the potential function used in \cite{ZavlanosCDC2007} goes to infinity at singularities. Note that the mobile robots in this work are nonholonomic while \cite{ZavlanosCDC2007,DongHuang15} addressed double integrator systems.
\end{remark}

\begin{remark}
The first sum in (\ref{proto.1sdesign.speed}) consists of the gradients of $U(q_i,q_j)$ and the unit vector $e(\theta_i)$ which are bounded by definition. The second sum in (\ref{proto.1sdesign.speed}) is comprised of $\sigma_1(.)$, which is a linear saturation function defined in Section \ref{problemformulation}. Hence, as a whole, the control law (\ref{proto.1sdesign.speed}) for each agent is bounded. This satisfies our objective on the boundedness of the control input $u_i$.
\end{remark}

Theorem \ref{thms} shows that the design (\ref{proto.1sdesign.speed}) achieves speed consensus and the goals G2) and G3). In the next subsection, we will design $\tau_i$ for orientation consensus completing the goal G1).

\subsection{Orientation Consensus}
Motivated by the orientation consensus design method presented in \cite{LiangLee2005}, we shall develop a bounded control approach which employs a saturation function in Section \ref{problemformulation}.

Define the orientation trajectory error for agent $i$ as
\be
	e_i=\theta_i-\theta_r,
\ee
where $\theta_r$ is the desired orientation of the flock. Thus, the angle difference between two agents $i$ and $j$ is
\be
	\theta_i-\theta_j=\theta_i-\theta_r-(\theta_j-\theta_r)=e_i-e_j.
\ee

Similarly to \cite{LiangLee2005}, the following lemma is employed for our convergence analysis.
\begin{lemma}
\label{lemmaei}
Suppose that the flock possesses a graph $\mathcal{G}(t)$, then the trajectory error signals of the group have the
following property:
  \begin{align}
  \label{lemmaei_equ}
    \frac{1}{2}\sum_{i=1}^N\sum_{j\in \mathcal{N}_i(t)}(e_i - e_j)(\dot{e}_i - \dot{e}_j) = \sum_{i=1}^N\sum_{j\in \mathcal{N}_i(t)}e_i(\dot{e}_i - \dot{e}_j).
  \end{align}
  \vspace{0pt}
\end{lemma}
\begin{proof}
The proof is similar to the one employed in Lemma \ref{lemmasum}.
\end{proof}

We have the following orientation consensus theorem.
\begin{theorem}\label{thmo}
 Assume that the desired orientation $\theta_r$ and its first and second derivation are bounded, and the collective system (\ref{mobiledynamics}) is subject to the following protocol 
\be\label{tauidesign}
	\tau_i=\ddot{\theta}_r-\sigma_2(\dot{\theta}_i-\dot{\theta}_r)-\frac{k_{\theta}}{n_i+1}[(n_i+1)\theta_i-\sum_{j\in \mathcal{N}_i(t)} \theta_j-\theta_r],
\ee
where $n_i$ is the number of the neighbors of robot $i$ and $k_{\theta}$ is a positive parameter. Then, all the mobile robots eventually reach consensus on the heading angles $\theta_i$ in the sense that
  \begin{align}
    \lim\limits_{t\to\infty}(\theta_i(t) - \theta_j(t)) = 0, \forall i, j.
  \end{align}
  \vspace{3pt}
\end{theorem}
\begin{proof}
Consider the following Lyapunov function candidate
\be
V_2(t)=\frac{1}{2}\sum_{i=1}^N\frac{k_{\theta}}{n_i+1}e_i^2+\frac{1}{2}\sum_{i=1}^N\dot{e}_i^2+\frac{1}{4}\sum_{i=1}^N\sum_{j\in \mathcal{N}_i(t)}\frac{k_{\theta}}{n_i+1}(\theta_i-\theta_j)^2.
\ee
According to Theorem \ref{thms}, there exists $t_k$ such that for $t\geq t_k$, $\mathcal{G}(t)=\mathcal{G}(t_k)$.
The derivative of $V_2(t)$ with respect to $t$ for $t\geq t_k$ is given as
\bea
\nn	\dot{V}_2(t)&=&\sum_{i=1}^N\frac{k_{\theta}}{n_i+1}e_i\dot{e}_i+\sum_{i=1}^N\dot{e}_i(\tau_i-\ddot{\theta}_r)+
\sum_{i=1}^N\sum_{j\in \mathcal{N}_i(t)}\frac{k_{\theta}}{n_i+1}(\theta_i-\theta_j)(\dot{\theta}_i-\dot{\theta}_j)\\
\nn &=&\sum_{i=1}^N\frac{k_{\theta}}{n_i+1}e_i\dot{e}_i-\sum_{i=1}^N\frac{k_{\theta}}{n_i+1}\dot{e}_ie_i-\sum_{i=1}^N\dot{e}_i\sigma_2(\dot{e}_i)\\
\nn&&-\sum_{i=1}^N\sum_{j\in \mathcal{N}_i(t)}\frac{k_{\theta}}{n_i+1}\dot{e}_i(\theta_i-\theta_j)+
\frac{1}{2}\sum_{i=1}^N\sum_{j\in \mathcal{N}_i(t)}\frac{k_{\theta}}{n_i+1}(\theta_i-\theta_j)(\dot{\theta}_i-\dot{\theta}_j)\\
\nn&=&-\sum_{i=1}^N\dot{e}_i\sigma_2(\dot{e}_i)-\sum_{i=1}^N\sum_{j\in \mathcal{N}_i(t)}\frac{k_{\theta}}{n_i+1}\dot{e}_i(e_i-e_j)\\
&&+\frac{1}{2}\sum_{i=1}^N\sum_{j\in \mathcal{N}_i(t)}\frac{k_{\theta}}{n_i+1}(e_i-e_j)(\dot{e}_i-\dot{e}_j).
\eea
Using Lemma \ref{lemmaei}, we obtain
\be\label{dotV2}
	\dot{V}_2(t)=-\sum_{i=1}^N\dot{e}_i\sigma_2(\dot{e}_i)\leq 0
\ee
since $\sigma_2(.)$ is a nondecreasing function defined in Section \ref{problemformulation}. Since $\theta_i,\theta_r \in [-pi,pi]$ and  $\ddot{\theta}_r$ is bounded, the control law (\ref{tauidesign}) implies $\ddot{e}_i$ is bounded. By the Barbalat's lemma, from (\ref{dotV2}), $\dot{e}_i\rightarrow 0$. Also, since $\ddot{\theta}_r$ is bounded, $\ddot{e}_i\rightarrow 0$. Therefore, the control law (\ref{tauidesign}) implies that $\theta_i\rightarrow \theta_j$, which proves the theorem.
\end{proof}

\begin{remark}
The boundedness of the control law (\ref{tauidesign}) is guaranteed by the properties of linear saturation function $\sigma_2$ and  the fact that $\theta_i,\theta_j,\theta_r \in [-\pi,\pi]$. This demonstrates that the proposed control scheme meets the requirement on the physical limits of the control inputs.
\end{remark}

\begin{remark}
It should be noted that our control law for orientation consensus is similar to the one in \cite{LiangLee2005} but here the boundedness of the control input is taken into account. The scheme in this chapter also shares the same objective as the one in \cite{Nguyen2016} but offers a more simple form and implementation.
\end{remark}

Combining Theorems \ref{thms} and \ref{thmo}, we have the following bounded flocking theorem.
\begin{theorem}
  Suppose that the collective system (\ref{mobiledynamics}) is subject to the bounded protocols (\ref{proto.1sdesign.speed})  and (\ref{tauidesign}). Suppose further that the initial configuration of the collective system  (\ref{mobiledynamics})  is such that $\mathcal{N}(0)$ is connected. Then, all the multiple flocking goals of velocity consensus, cohesion maintenance, and collision avoidance are achieved.
  \vspace{3pt}
\end{theorem}
\begin{proof}
 The proof is straightforward from the results of Theorems 3.1 and 3.2.
\end{proof}

\section{Avoidance of Obstacles}
\label{avoidance_obstacles}
The problem of obstacle avoidance has been extensively studied in the literature \cite{LiangLee2005,LiangLee2006,Nguyen2016}. In this section, we employ the idea from \cite{LiangLee2005} to derive our control algorithm in which the agents are able to pass obstacles. It is shown that a convex obstacle can be extrapolated by a round shape or a rectangle \cite{LiangLee2005,Nguyen2016}. In \cite{LiangLee2005}, a convex obstacle is presented by a circle, which is used in our work. Let $q_r=[x_r,y_r]^T$ be the coordinate of the robot, $q_{jobs}=[x_{obs},y_{obs}]^T$ be the projection point of the robot onto obstacle $j=1,2,\dots, n_{obs}$ where $n_{obs}$ is the number of obstacles, and $\bar{O}_k=[x_k,y_k]^T$ be the centre of the obstacle. From \cite{LiangLee2005},
\be
	q_{jobs}=\frac{r}{\|q_r-\bar{O}_k\|}q_r+(1-\frac{r}{\|q_r-\bar{O}_k\|})\bar{O}_r
\ee
where $r$ is the radius of the obstacle. The projection point has the following velocity
\be
	v_{jobs}=\frac{vr\sin{\alpha}}{\|q_r-q_{jobs}\|}
\ee
where $v$ is the velocity of the agent, $\alpha$ is the angle between the heading of the robot and the straight line which connects the robot to the centre of the obstacle. The projection point moves in the direction \cite{LiangLee2005}
\bea
	\theta_{jobs}&=&-\frac{\pi}{2}+\alpha+\theta, \, \alpha>0\\
	\theta_{jobs}&=&\frac{\pi}{2}+\alpha+\theta, \, \text{otherwise}.
\eea
The fact that the projection point possesses a position, velocity, and orientation enables it to be an agent. The above descriptions of the robot and obstacle are demonstrated in Figure \ref{obs_robot}.
\begin{figure}[htp]
\centering
\includegraphics[width=10cm,height=10cm,keepaspectratio]{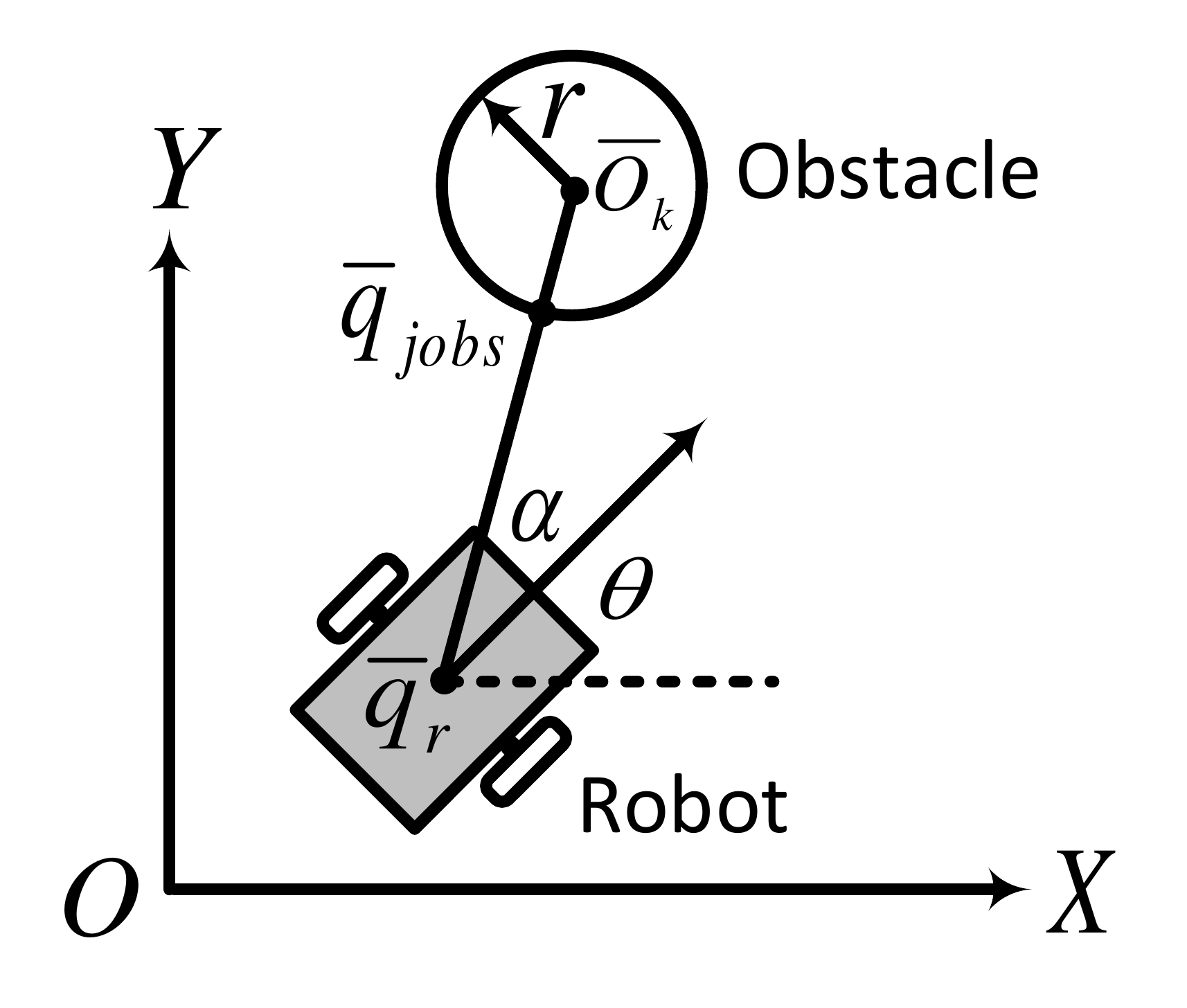}
\caption{Illustration of a robot and a convex obstacle.} \label{obs_robot}
\end{figure}

We have the following orientation consensus theorem.
\begin{theorem}\label{thmobs}
The following control protocol guarantees that the robot avoids obstacles with arbitrary boundary shapes:
\be
\label{obs.speed}
  u_i  = -\sum_{j\in \mathcal{N}_{obs}} \nabla_{q_i} U(q_i,q_j)\cdot e(\theta_i) - \sum_{j\in \mathcal{N}_{obs}} \sigma_1(v_i - v_j) 
\ee
\be\label{obs_heading}
	\tau_i=\ddot{\theta}_r-\sigma_2(\dot{\theta}_i-\dot{\theta}_r)-\frac{k_{\theta}}{n_{obs}+1}[(n_{obs}+1)\theta_i-\sum_{j\in \mathcal{N}_i(t)} \theta_j-\theta_r],
\ee
where  $\mathcal{N}_{obs}$ is the set of obstacles, $n_{obs}$ is the number of obstacles and
\be
\theta_r=\frac{1}{n_{obs}}\sum_{j\in \mathcal{N}_{obs}}\theta_{jobs}.
\ee
  \vspace{3pt}
\end{theorem}
\begin{proof}
The proof can be derived using the same approach as in \cite{LiangLee2005}.
\end{proof}
\begin{remark}
It should be noted that the speed control law in (\ref{obs.speed}) enjoys the boundedness due to the saturation function $\sigma_1(.)$, which is different from the one in \cite{LiangLee2005}. Since $\theta_i,\, \theta_r \in(-\pi,\pi]$ and the saturation function $\sigma_2(.)$ is bounded, (\ref{obs_heading}) reveals that the heading control law is also bounded.
\end{remark}
\section{SIMULATION}
\label{simulation}

We conducted simulation for a multi-agent system of 15 mobile robots of the model (\ref{mobiledynamics}). A bump function is used to generate the smooth coordination function $U$. As the control (\ref{proto.1sdesign.speed}) invokes the gradient forces $\nabla_{q_i}U$, we designed the coordination function in the form
\begin{align}
  U(r) & = \int_0^r\varphi(s)ds
\end{align}
where $\varphi$ is a compact support function given by
\begin{align}
  \varphi(s) & = \begin{cases}
    p_1\exp\Big(\frac{-(s-s_0)^2}{((a-r_0)/2)^2- (s - s_0)^2}\Big) & \text{~if~} s \in (r_0,a) \nonumber\\
    p_2\exp\Big(\frac{-(s-s_1)^2}{((R_0-A)/2)^2- (s - s_1)^2}\Big) & \text{~if~} s \in (A,R_0) \nonumber\\
    0 & \text{~otherwise}
  \end{cases}
\end{align}
where $ s_0=\frac{r_0+a}{2}$, $ s_1=\frac{R_0+A}{2}$, and $p_1, p_2, a, A, r_0$ and $R_0$ are design parameters.

The parameters of the coordinate function are $r_0 = 1$, $a = 3$, $A = 6$, $R_0 = 8$, and $U_M=15$. The parameter for the control law (\ref{tauidesign}) is $k_{\theta}=1.5$. The initial positions of 15 mobile robots are randomly distributed on three circles. Their coordinates are
\bea
\nn x(i)&=& \Gamma\,\sin(z\,\pi\,(i-1)/\Gamma+\pi)\\
\nn y(i) &=& \Gamma\,\cos(z\,\pi\,(i-1)/\Gamma+\pi) 
\eea
where $z=1$, and
\be
\nn	\Gamma=\begin{cases}4 & \mbox{if } i<6\\ 8& \mbox{otherwise}\end{cases}.
\ee
The initial values of $\theta_i$ and $v_i$ are randomly chosen. The initial value of $w_i$ is 0. The desired orientation of the flock is $\theta_r=\pi/2$.

We obtained the simulation results shown in Figures \ref{flocking}--\ref{tau}. It is shown in Figures \ref{orientation} and \ref{angularspeed}, the heading angles converge to $\theta_r$ and angular speeds of all agents converge to 0 after $t=40 s$. The linear speeds are depicted in Figure \ref{speed}, where the convergence of all agents takes place after $t=45 s$. In Figure \ref{angularspeed}, the angular speeds converge after 40 seconds. Hence,  Figure \ref{orientation} and Figure \ref{speed} demonstrate that consensuses on orientation and linear speed of the mobile robots have been obtained. The speed control and steering control signals are shown to be bounded in Figure \ref{u} and Figure \ref{tau}. The minimum distance among agents is described in Figure \ref{dmin}, which shows that collision avoidance is guaranteed. The flocking behavior is shown in Figure \ref{flocking}, where no collision occurred.

\begin{figure}[htp]
\centering
\includegraphics[width=10cm,height=10cm,keepaspectratio]{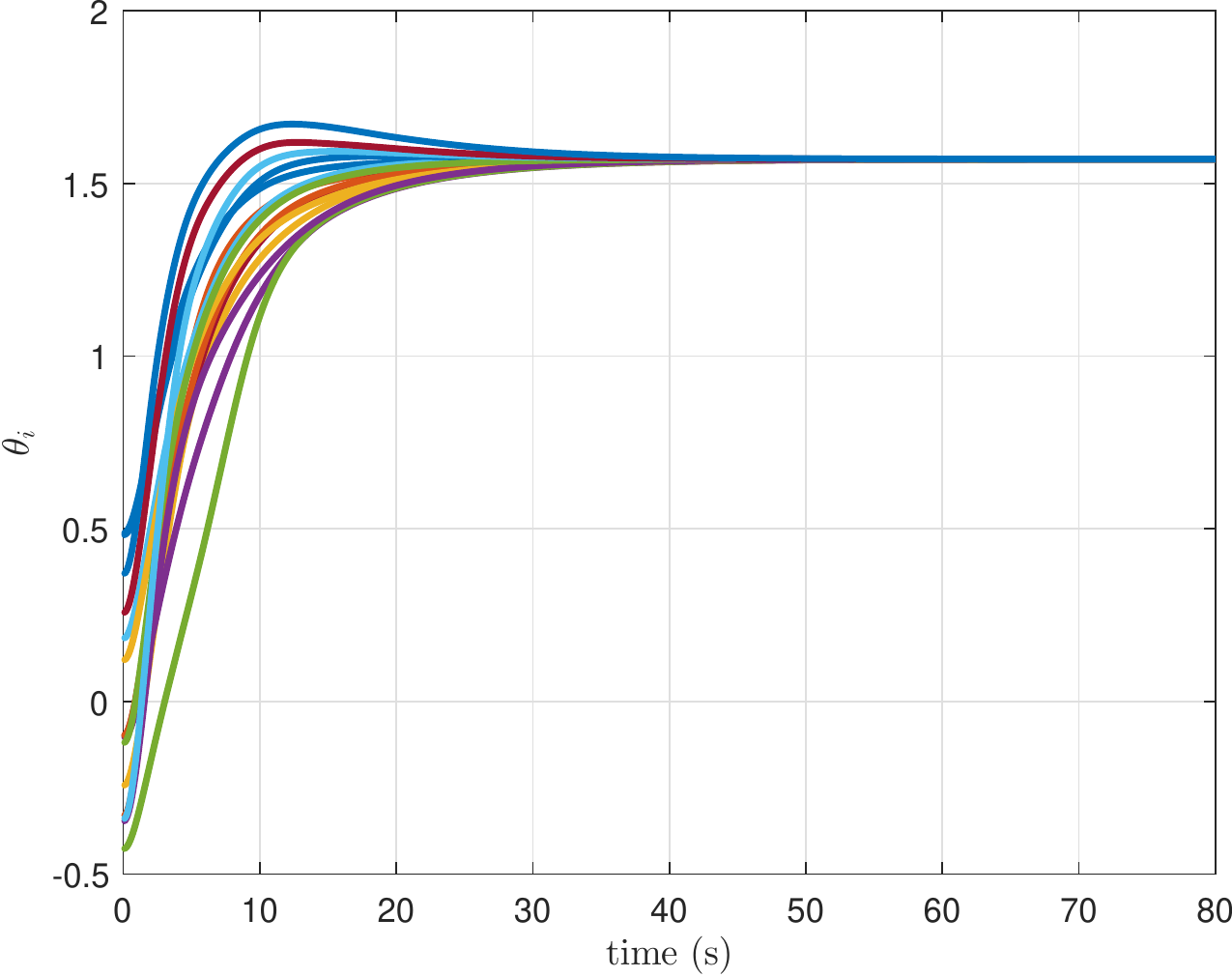}
\caption{Orientation consensus} \label{orientation}

\end{figure}

\begin{figure}[htp]
\centering
\includegraphics[width=10cm,height=10cm,keepaspectratio]{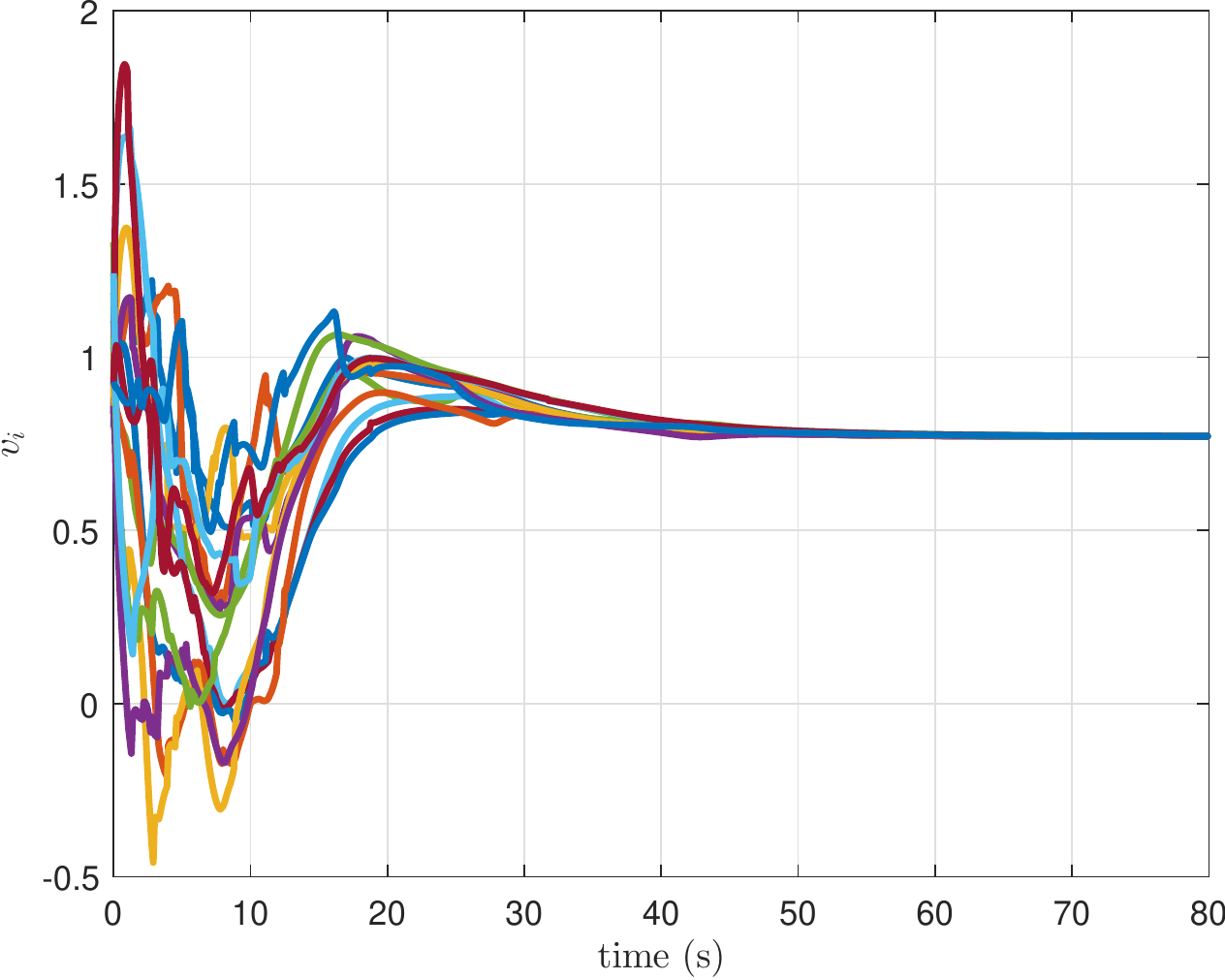}

\caption{Linear speed consensus} \label{speed}
\end{figure}

\begin{figure}[htp]
\centering
\includegraphics[width=10cm,height=10cm,keepaspectratio]{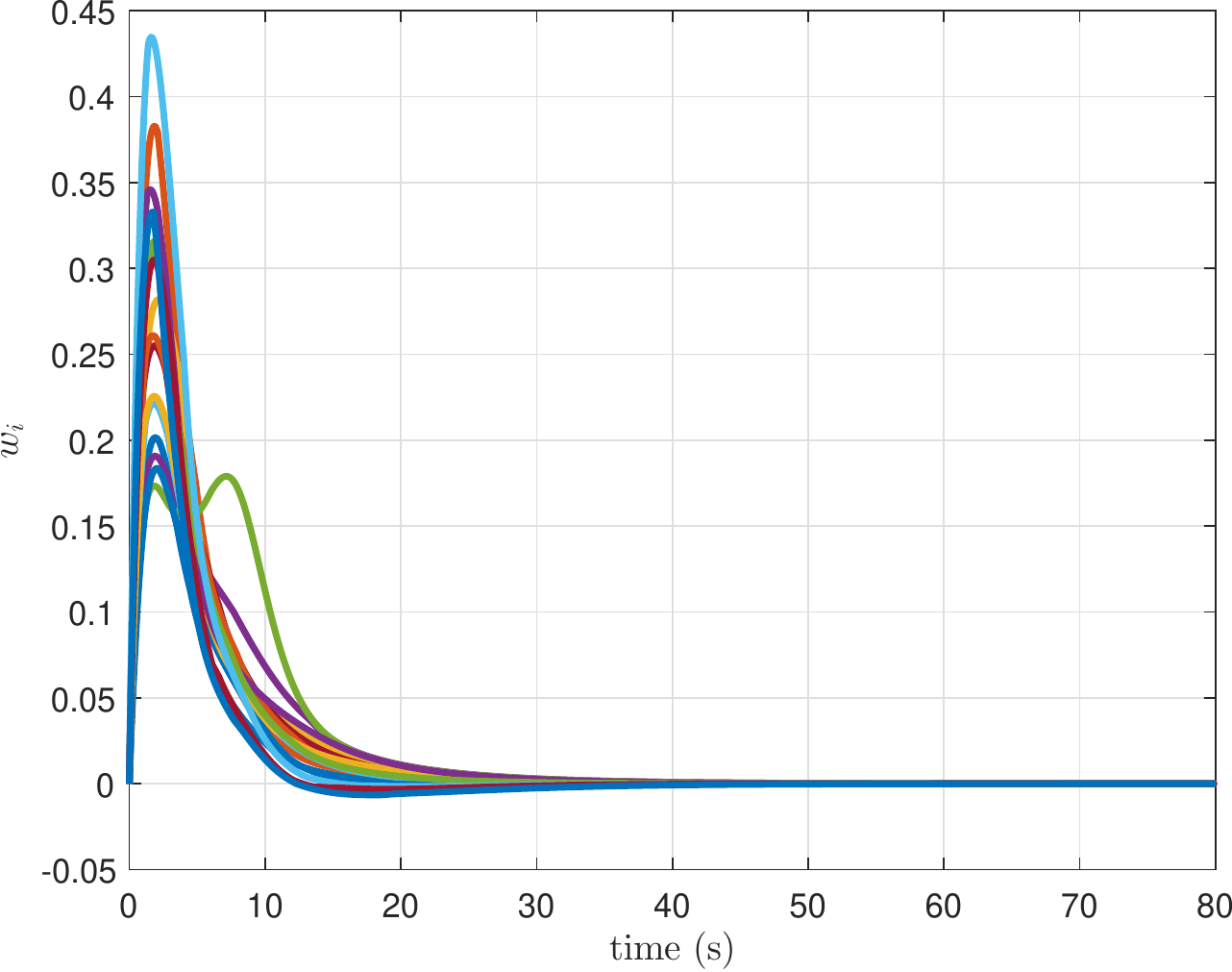}

\caption{Angular speed} \label{angularspeed}
\end{figure}

\begin{figure}[htp]
\centering
\includegraphics[width=10cm,height=10cm,keepaspectratio]{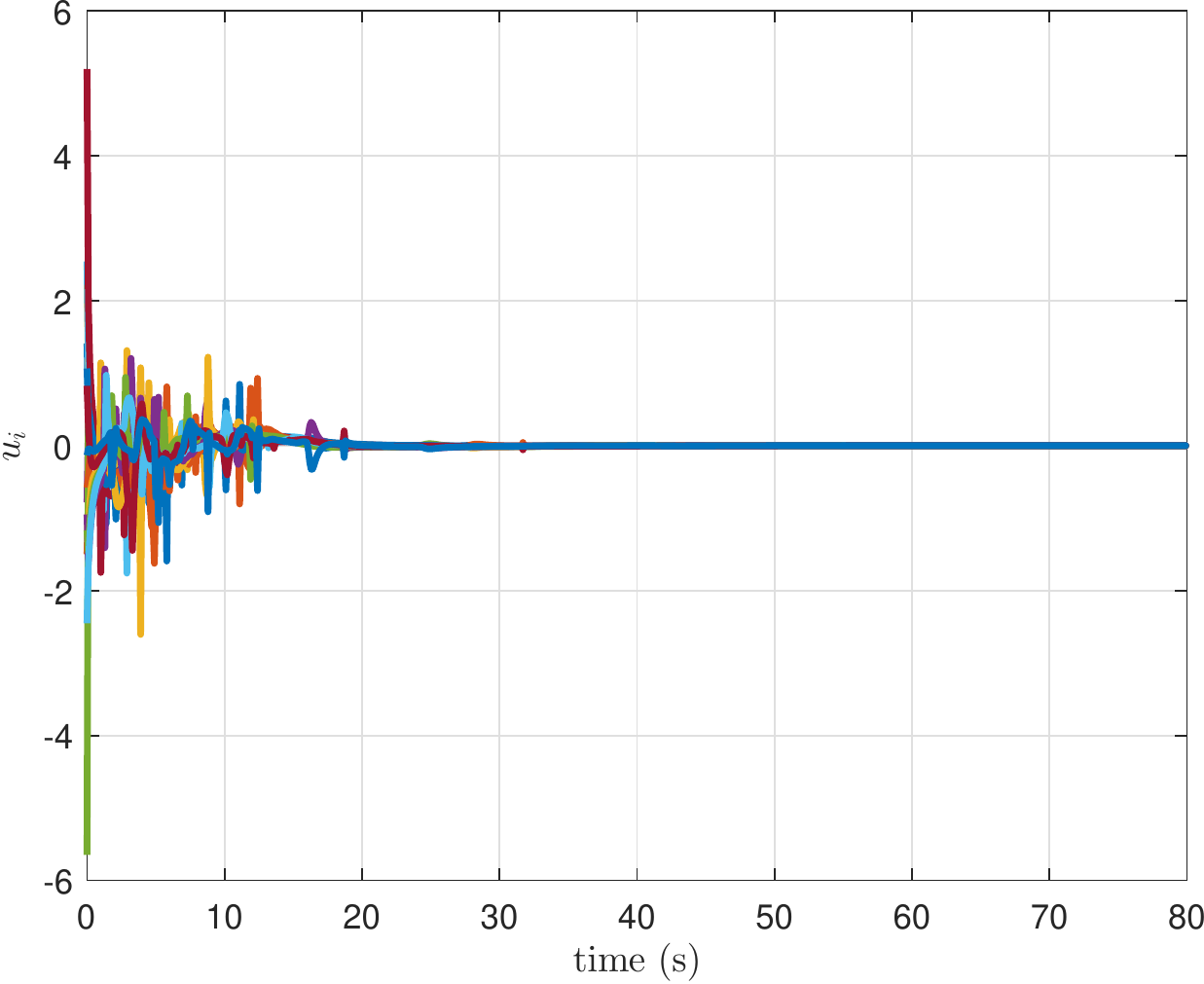}

\caption{Translational force} \label{u}
\end{figure}

\begin{figure}[htp]
\centering
\includegraphics[width=10cm,height=10cm,keepaspectratio]{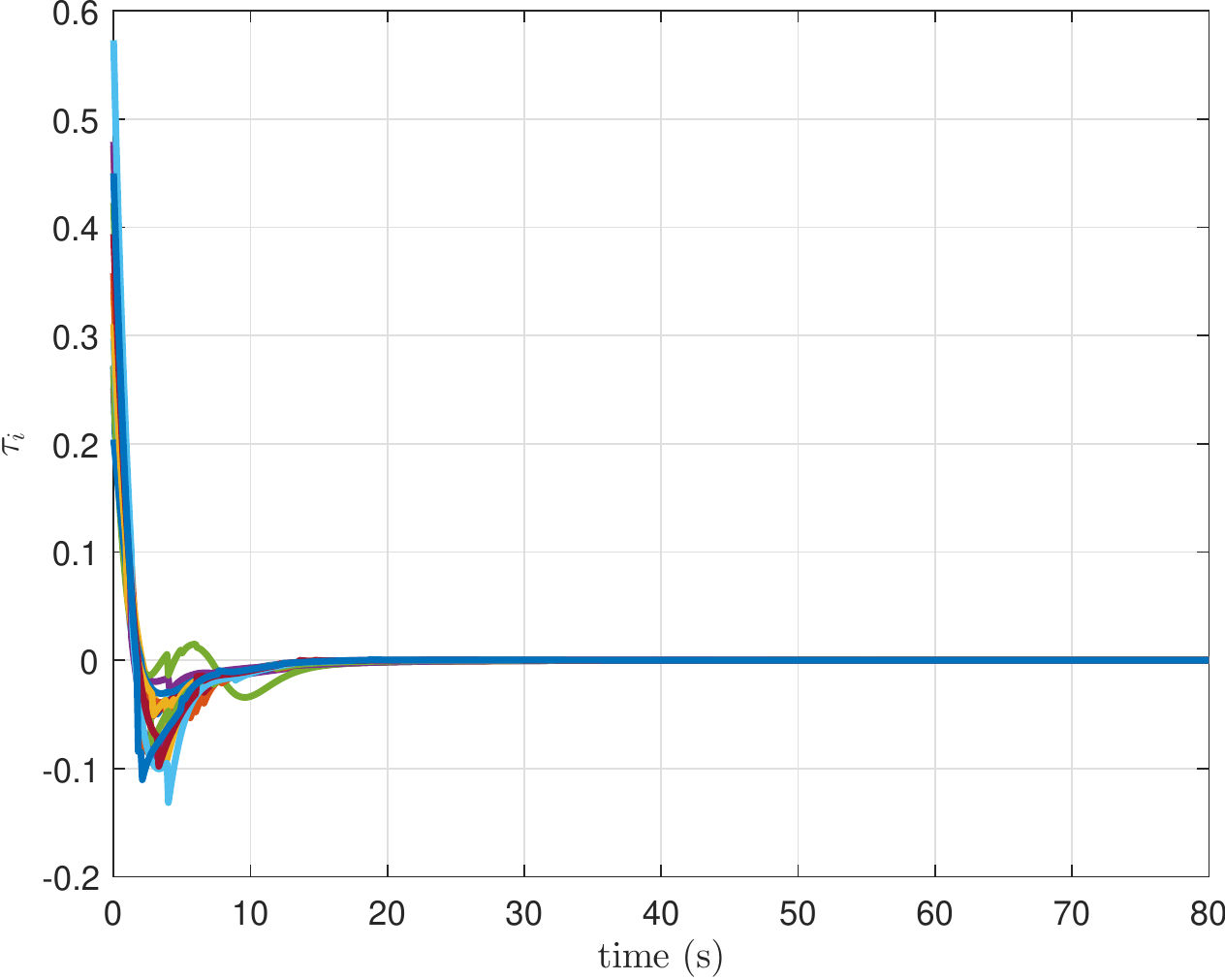}
\caption{Steering control} \label{tau}

\end{figure}

\begin{figure}[htp]
\centering
\includegraphics[width=10cm,height=10cm,keepaspectratio]{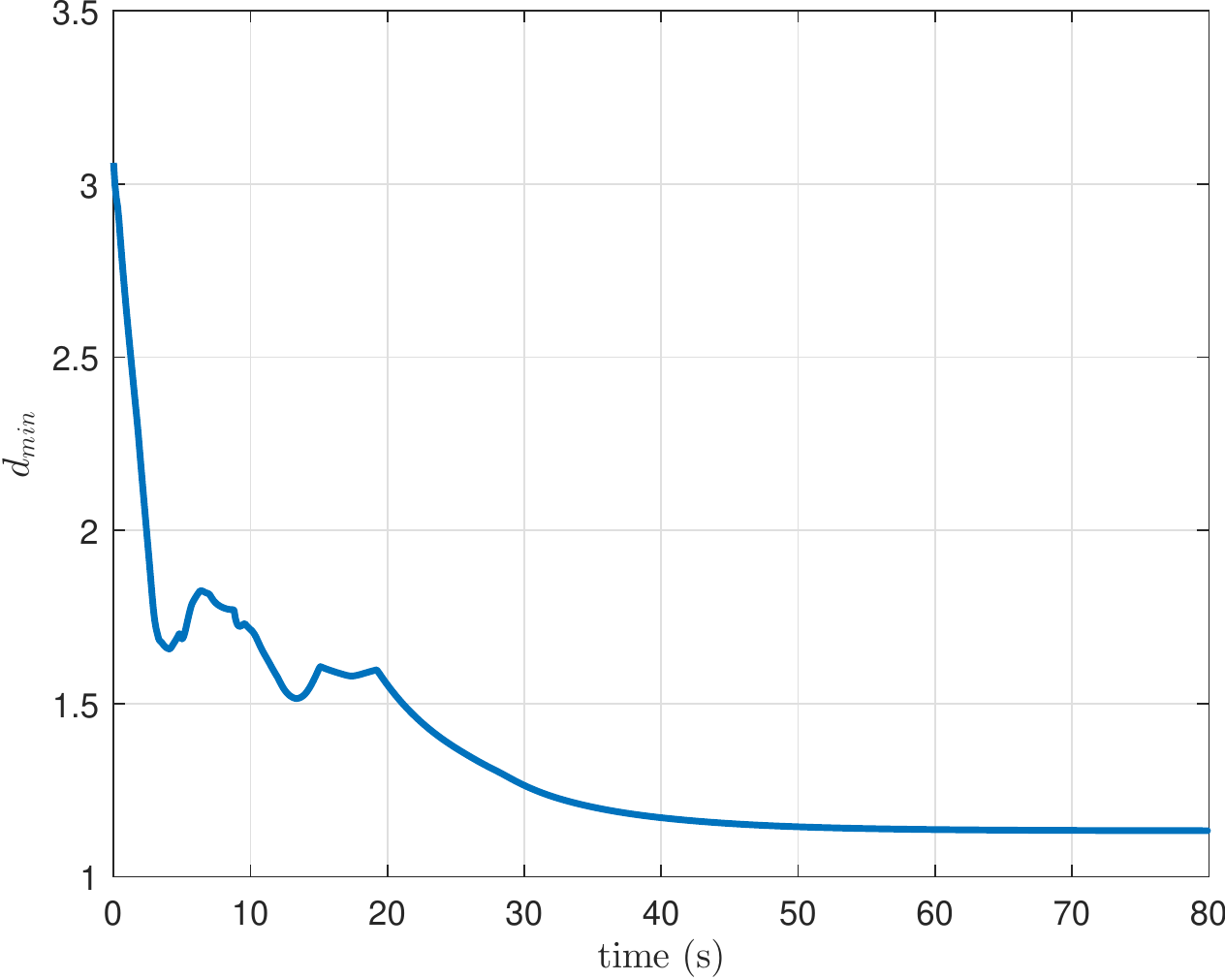}
\caption{The minimum distance among agents} \label{dmin}

\end{figure}

\begin{figure}[htp]
\centering
\includegraphics[width=15cm,height=15cm,keepaspectratio]{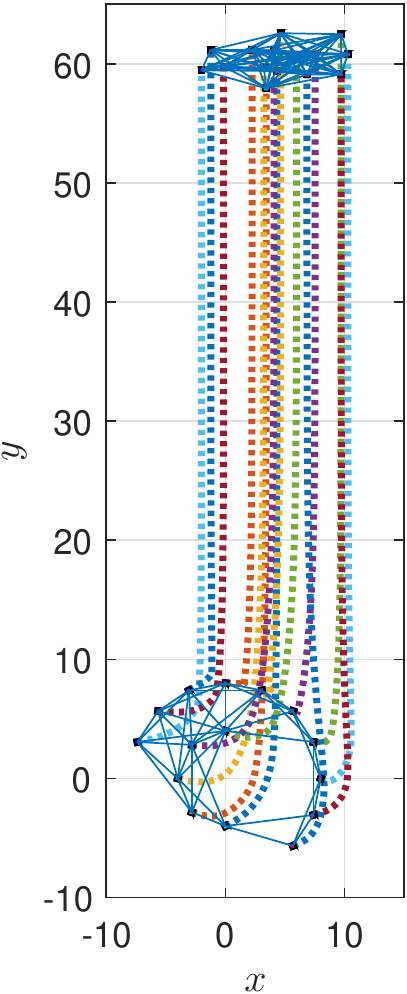}
\caption{Distributed flocking of 15 mobile robots} \label{flocking}
\end{figure}

Next, we consider a simulation for a multiagent system with an obstacle. The obstacle is a circle whose coordinate is $[12,-1]$ with a radius of 1. The configuration of the multiagent system is the same as in the obstacle-free case above. The desired group heading is chosen as $\theta_r=\pi/4$. An agent is chosen as a leader of the group. When the collision avoidance mechanism of the leader is inactive, a speed control law is designed to drive the group at a constant speed, that is 
\be
	u_l=- \sigma_1(v_l - v_r)
\ee
where $u_l$ is the speed control of the leader, $v_l$ is its linear speed, and $v_r=0.2$.
During the first 50 seconds of their evolution, the mobile robots encounter the obstacle. The control laws (\ref{obs.speed}) and (\ref{obs_heading}) enable them to avoid potential collisions with the obstacle and with other neibouring agents. The orientations of the robots in Figure \ref{orientationobs} converge to $\pi/4$ after 60 seconds. Similarly, in Figure \ref{speedobs}, the linear speeds of the agents converge after 50 seconds. In Figure \ref{angularspeedobs}, the angular speeds converge faster to 0 after 50 seconds. Figs. \ref{uobs} and \ref{tauobs} demonstrate that the control inputs both are bounded. Finally, Figure \ref{dminobs} shows that no collision takes place during the evolution of the robots. The evolution of the multiagent system is shown in Figure \ref{flockingobs} in which the robots cooperate to form a flocking and avoid the obstacle.

\begin{figure}[htp]
\centering
\includegraphics[width=10cm,height=10cm,keepaspectratio]{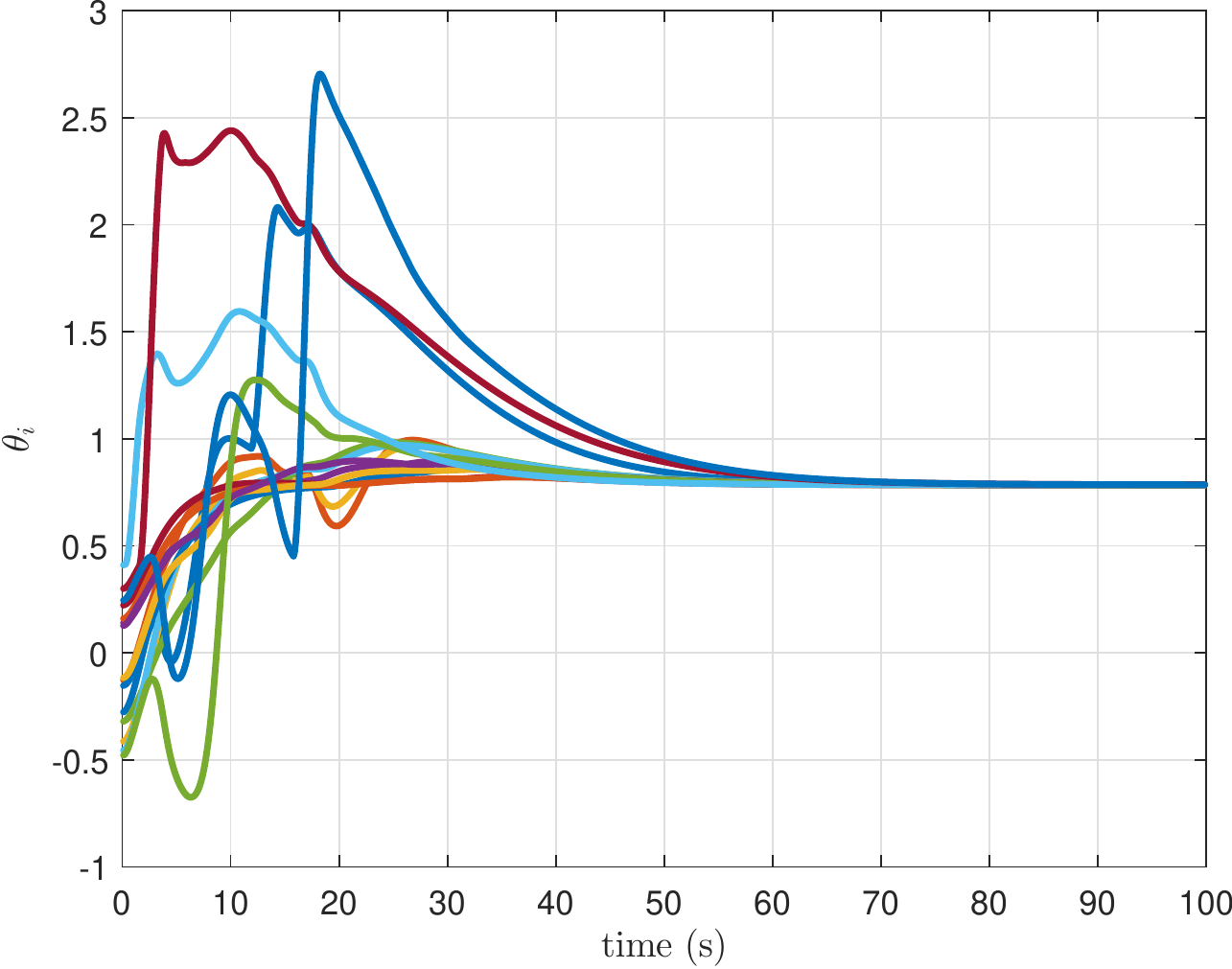}
\caption{Orientation consensus for the obstacle avoidance case} \label{orientationobs}

\end{figure}

\begin{figure}[htp]
\centering
\includegraphics[width=10cm,height=10cm,keepaspectratio]{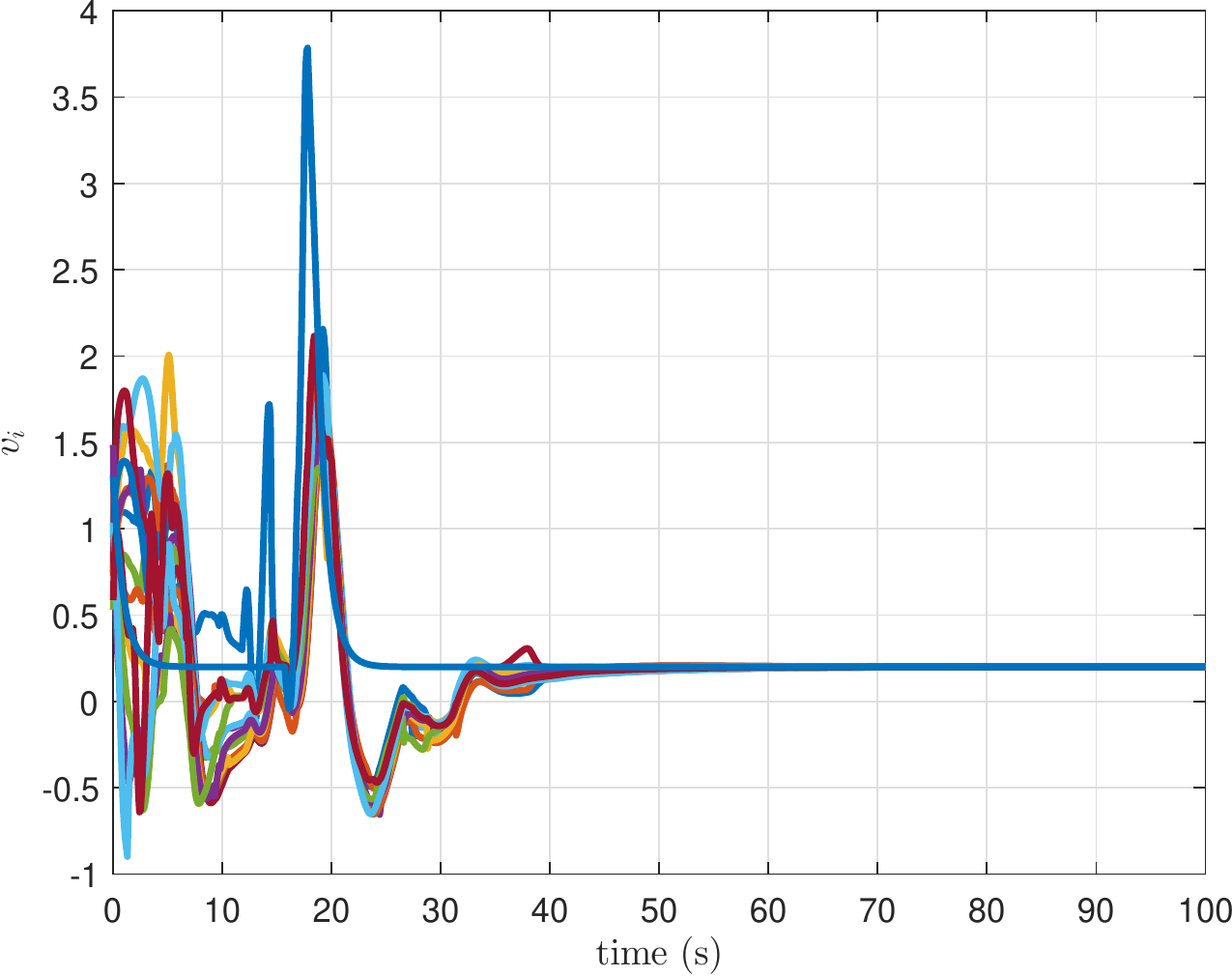}

\caption{Linear speed consensus for the obstacle avoidance case} \label{speedobs}
\end{figure}

\begin{figure}[htp]
\centering
\includegraphics[width=10cm,height=10cm,keepaspectratio]{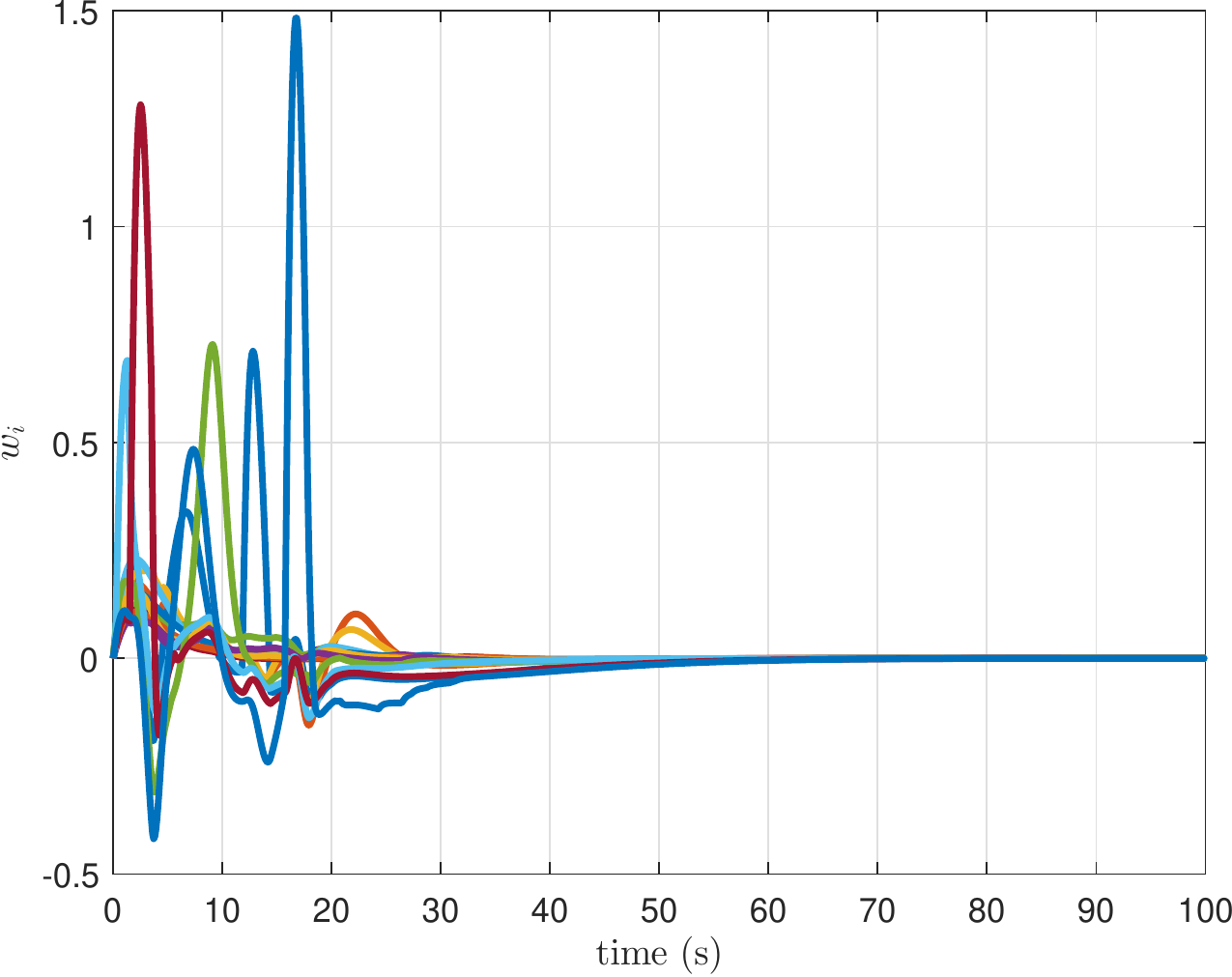}

\caption{Angular speed for the obstacle avoidance case} \label{angularspeedobs}
\end{figure}

\begin{figure}[htp]
\centering
\includegraphics[width=10cm,height=10cm,keepaspectratio]{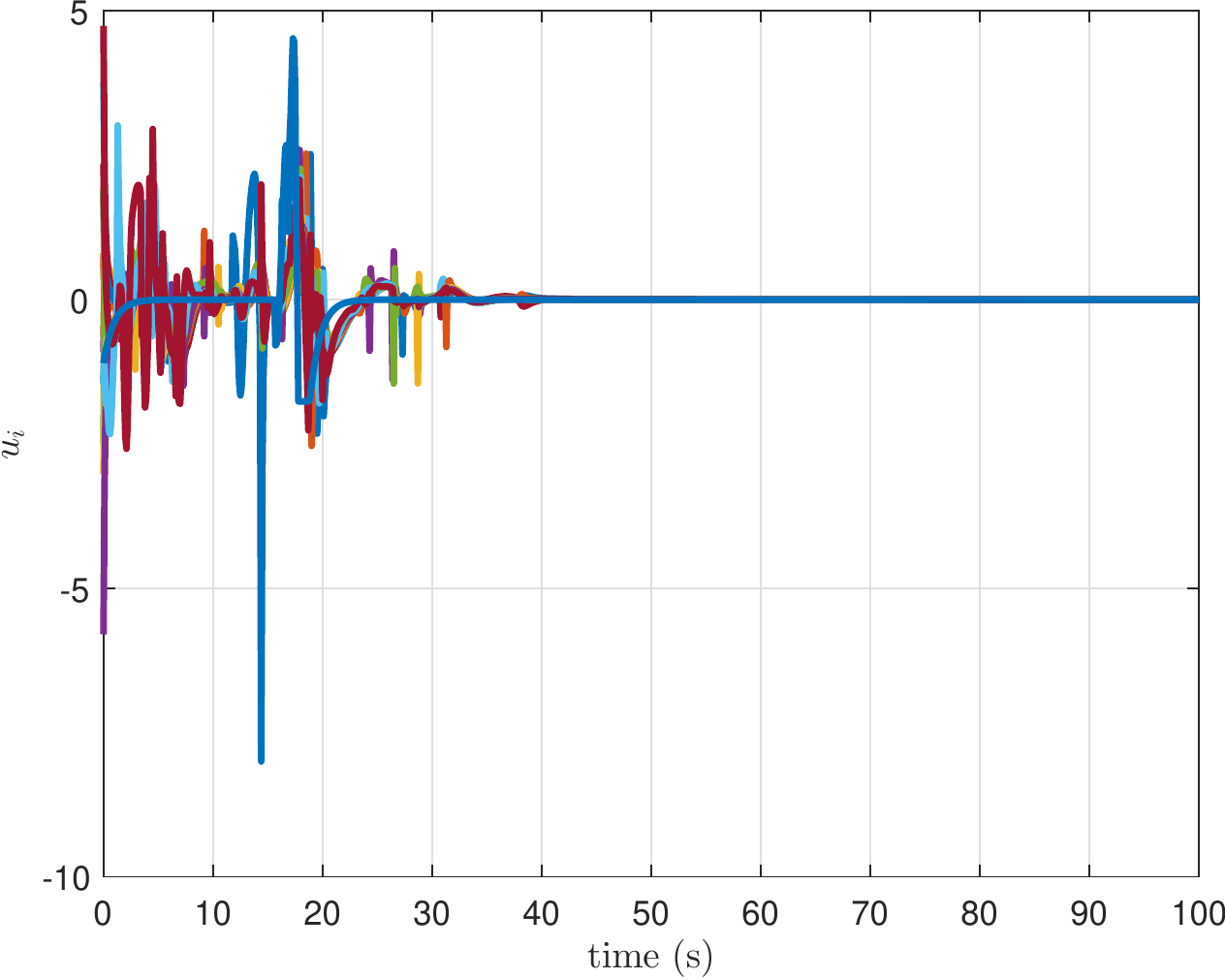}

\caption{Speed control for the obstacle avoidance case} \label{uobs}
\end{figure}

\begin{figure}[htp]
\centering
\includegraphics[width=10cm,height=10cm,keepaspectratio]{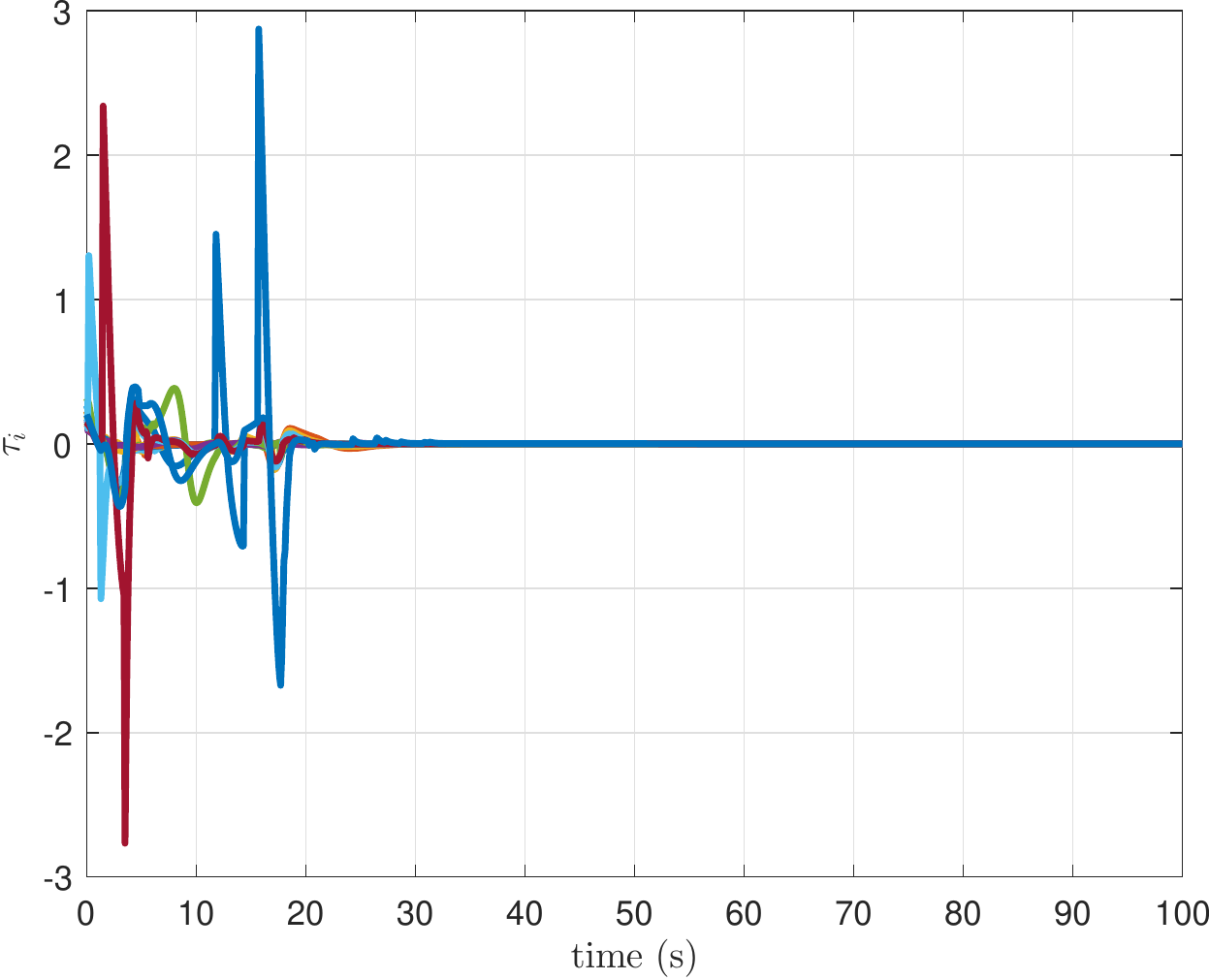}
\caption{Steering control for the obstacle avoidance case} \label{tauobs}

\end{figure}

\begin{figure}[htp]
\centering
\includegraphics[width=10cm,height=10cm,keepaspectratio]{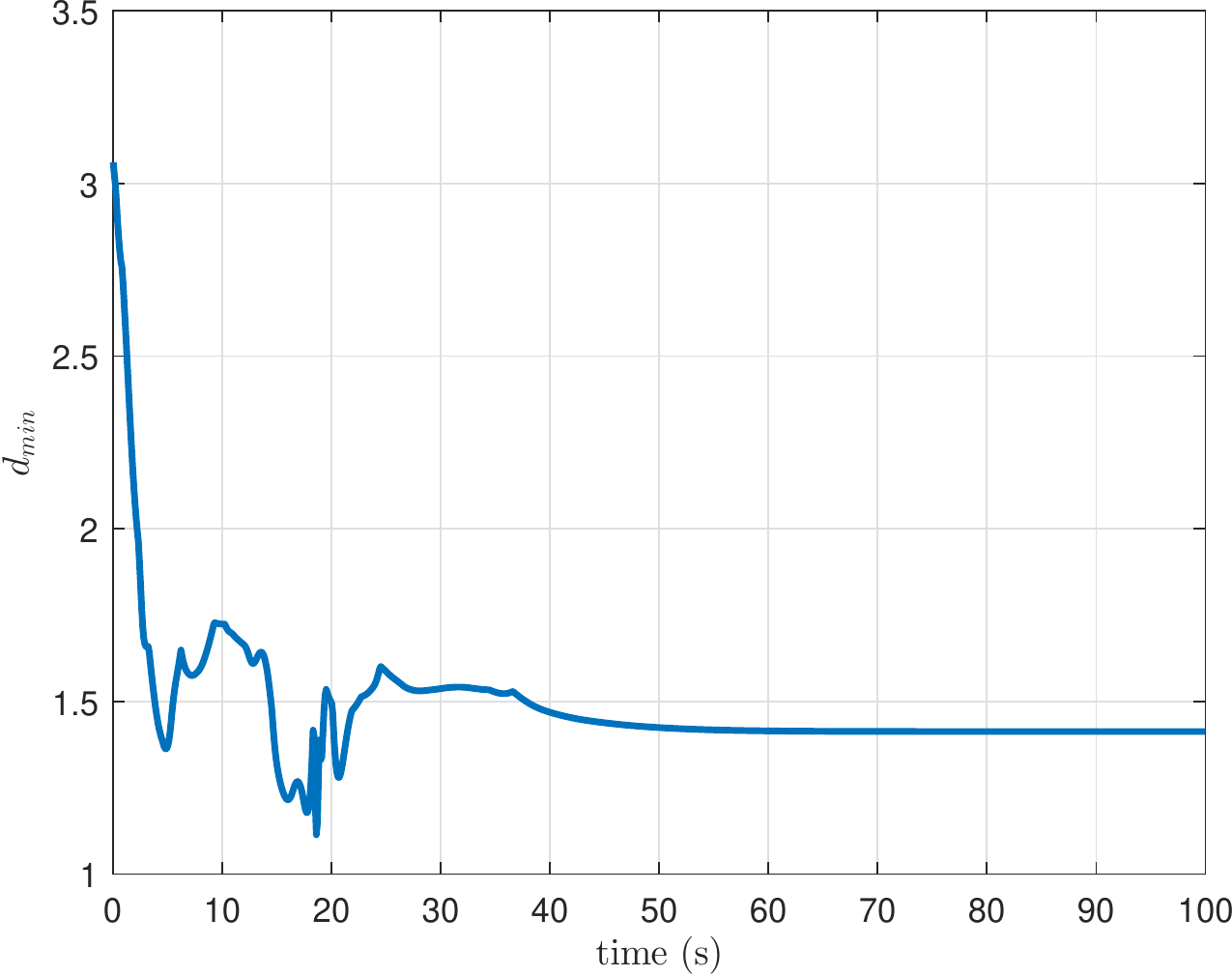}
\caption{The minimum distance among agents for the obstacle avoidance case} \label{dminobs}

\end{figure}

\begin{figure}[t]
\centering
\includegraphics[width=11cm,height=11cm,keepaspectratio]{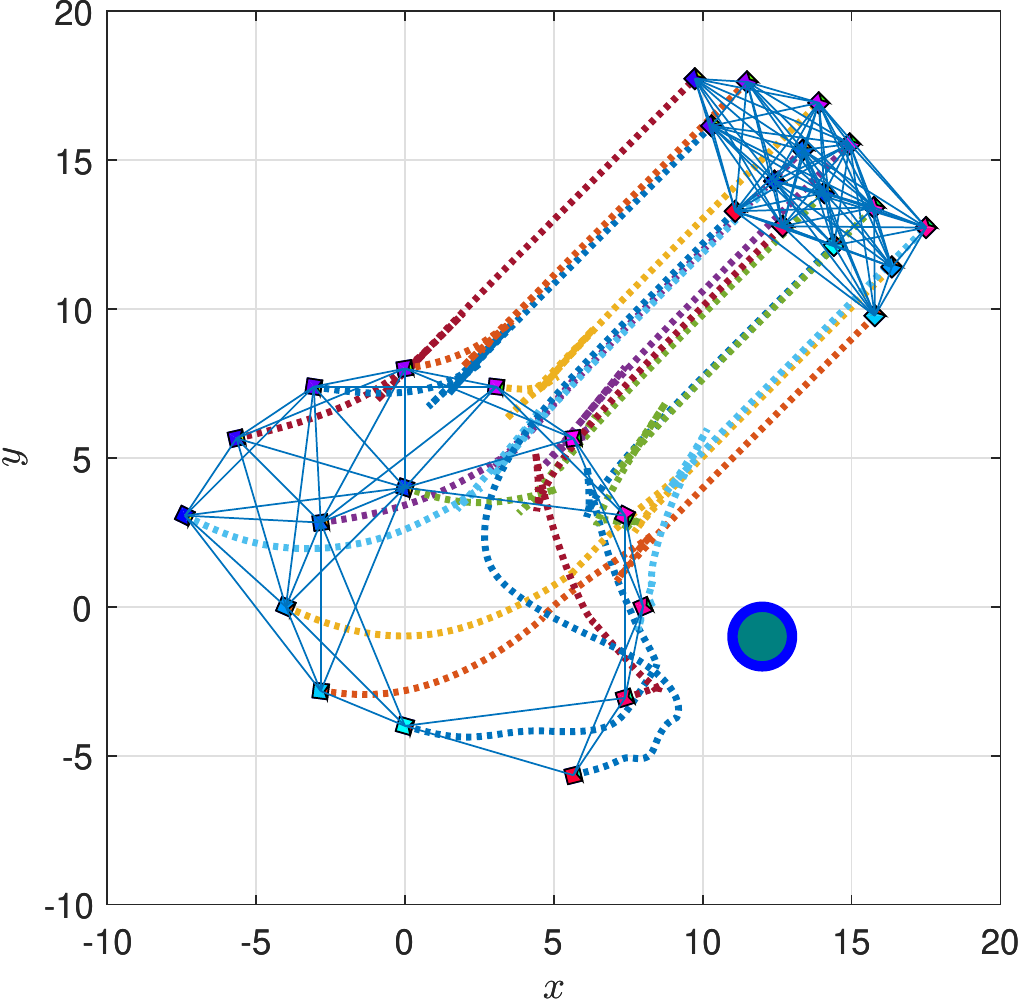}
\caption{Distributed flocking of 15 mobile robots for the obstacle avoidance case} \label{flockingobs}
\end{figure}

\section{CONCLUSIONS}
\label{conclusions}

This chapter has presented a bounded decentralized control protocol for the flocking problem of mobile robots by a systematic fashion, where the control laws only require the information of neighbor agents. The proposed scheme is modular in designing the speed control and steering control separately. Theoretical and numerical results have shown that using our proposed method, a collective system of mobile robots achieves all the multiple objectives of the flocking control: velocity consensus, cohesion maintenance, and collision avoidance. 

Future work would consider the shape and size of each mobile robot and obstacle avoidance for which a similar context was studied in \cite{nguyen2016formation}. Noisy and uncertain environments can affect the performance of the proposed scheme. Robustness analysis and improved methods can be proposed to address this issue.

\section{Acknowledgement}
This work was partially supported by the National Science Foundation  under grant NSF-NRI \#1426828,  the National Aeronautics and Space Administration (NASA)  under Grant No. NNX15AI02H issued through the Nevada NASA Research Infrastructure Development Seed Grant, and the University of Nevada, Reno.


\bibliographystyle{vancouver-modified}
\bibliography{IEEEabrv,bookchapter}

\end{document}